\documentclass{opt2026} 
\usepackage{subcaption}

\usepackage{booktabs}
\usepackage{algorithm}
\usepackage{algorithmic}
\newtheorem{assumption}{Assumption}

\title[Curvature-Aware Clipping for DPSGD]{Curvature-Aware Clipping: Mitigating Bias in DPSGD via Approximate Second-Order Information}

\optauthor{%
\Name{Xincheng Xu} \Email{xincheng.xu@anu.edu.au}\\
\addr Australian National University, Australia
\AND
\Name{Thilina Ranbaduge} \Email{thilina.ranbaduge@csiro.au}\\
\addr CSIRO, Australia
\AND
\Name{Thierry Rakotoarivelo} \Email{thierry.rakotoarivelo@csiro.au}\\
\addr CSIRO, Australia
\AND
\Name{David Smith} \Email{david.smith@csiro.au}\\
\addr CSIRO, Australia
}

\begin{document}

\maketitle

\begin{abstract}%
Differentially Private Stochastic Gradient Descent (DPSGD) has become the de facto standard for training deep neural networks with rigorous privacy guarantees. However, the utility of DPSGD is constrained by gradient clipping bias and DP noise. While existing work has tried to address these issues, the geometric misalignment caused by clipping bias remains a critical challenge. To bridge this gap, we propose a novel framework, \textit{Curvature-Aware Differentially Private Stochastic Gradient Descent (CA-DPSGD)}, which leverages approximate second-order information to construct an optimal, curvature-aware clipping space. CA-DPSGD integrates a rectified estimator for the curvature matrix and is coupled with an anisotropic noise mechanism to ensure privacy guarantees. Extensive empirical experiments demonstrate that our approach outperforms state-of-the-art clipping methods.
\end{abstract}


\section{Introduction}
\label{sec:intro}

Differentially Private Stochastic Gradient Descent (DPSGD)~\citep{abadi2016} is the most widely adopted optimization algorithm used for training deep models with
Differential Privacy (DP). DPSGD enforces ($\epsilon$,$\delta$)-DP through per-sample gradient clipping and calibrated Gaussian noise. However, DPSGD results in performance degradation compared to standard non-private training.

Different DPSGD variants have been proposed to mitigate these limitations, covering two main directions: (1) reducing DP noise variance from a temporal~\citep{zhang2024,zhang2025} or geometric perspective~\citep{zhou2020,yu2021a,asi2021}, and (2) mitigating clipping bias through a temporal perspective~\citep{xiao2023,zhang2023}. However, existing methods often suffer from complicated hyperparameter tuning, excessive computational and memory costs, and a lack of theoretical guarantees. 

To bridge this gap, we propose Curvature-Aware Differentially Private Stochastic Gradient Descent (CA-DPSGD), a novel method that reduces clipping bias from a geometric perspective. We first derive the optimal clipping space to minimize clipping bias while maintaining the same amount of DP noise compared to vanilla DPSGD by solving an optimization problem. We construct the clipping space in each training iteration via a private approximation of the second-order moments of the gradients. Crucially, to reduce the effect of DP noise, we design a DP variance-reduced estimator to correct the DP noise contribution in a privatized batch-gradient second-moment, coupled with an anisotropic noise injection mechanism to ensure DP guarantees. \textbf{Our main contributions are summarized as follows:}

\begin{itemize}
\item We formulate an optimization problem to characterize the clipping geometry under a bias–noise trade-off and derive the corresponding curvature-aware clipping spaces.
\item We design a noise correction mechanism and an adaptive scaling factor to construct a scalable diagonal clipping geometry using only previously privatized gradients.
\item  We establish privacy and convergence guarantees, evaluate CA-DPSGD on image and sentence classification benchmarks, and demonstrate its compatibility with existing DPSGD variants.
\end{itemize}

\section{Related Work}
\label{sec:related}

Existing literature on improving the utility of DPSGD~\citep{abadi2016} has explored adaptive clipping thresholds~\citep{andrew2021,bu2024,xia2023}, privacy budget allocation~\citep{lee2018,chen2023b}, geometric projection~\citep{zhou2020,yu2021a,asi2021,gilani2025}, filtering~\citep{tang2024,zhang2024,zhang2025}, and clipping bias mitigation~\citep{xiao2023,zhang2023,bethune2023}. However, they may require additional tuning, historical states, structured covariance information, or specialized model architectures. In contrast, our method constructs a lightweight diagonal clipping geometry from previously privatized gradients and explicitly corrects the known contribution of DP noise in the second-moment estimate. 

In non-private optimization, second-order methods have demonstrated improved optimization by leveraging curvature information. Adam~\citep{kingma2015} utilizes diagonal second-order gradient statistics, while Shampoo~\citep{gupta2018}, AdaHessian~\citep{yao2021}, and K-FAC~\citep{martens2015} employ structured or block-diagonal curvature approximations.  However, adapting these methods to the private setting is challenging, as the required per-sample statistics need to be protected, which degrades the accuracy of curvature estimation. Unlike approaches that use curvature information to modify the optimizer update, we leverage approximate second-order information to construct the clipping geometry prior to noise injection. More detailed discussions are provided in Appendix~\ref{app:related}.

\section{Preliminaries}
\label{sec:prem}

In this work, we focus on the supervised deep learning setting. The goal is to train a model with parameters $x \in \mathbb{R}^d$ by minimizing an empirical loss function over a dataset $D$ containing $N$ data points. This is formally expressed as the Empirical Risk Minimization (ERM) problem:
$\min_{x \in \mathbb{R}^d} f(x) = \frac{1}{N}\sum_{\xi \in D}f(x;\xi),$
where $f(x; \xi)$ represents the loss function for a single data point $\xi$, and $d$ represents the dimensionality of the model parameters. We use $(\cdot)_t$ to denote the variables in training iteration $t$ and denote the lower bound as $f^{*} = \text{inf } f(x)$. 

The geometry of the parameter space is formally characterized by the Fisher Information Matrix (FIM)~\citep{amari1998}, defined as $F(x) = \mathbb{E}_{\xi}[\nabla f(x; \xi) \nabla f(x; \xi)^\top]$. At iteration $t$, we write $F_t \equiv F(x_t)$. Under the negative log-likelihood setting, this quantity provides a local second-order characterization of the optimization landscape.

DPSGD~\citep{abadi2016} samples a mini-batch $\mathcal{B}_t$ of size $B$,
clips each per-sample gradient to threshold $C$, and adds Gaussian noise
to the averaged clipped gradient:
$
    \bar{g}_t
    =
    \frac{1}{B}
    \sum_{\xi\in\mathcal{B}_t}
    g_{t,\xi}
    \min\left\{1,\frac{C}{\|g_{t,\xi}\|_2}\right\}
    +w_t,\
    w_t\sim\mathcal{N}(0,\sigma^2I_d),
$
where $\sigma$ denotes the standard deviation of the Gaussian noise added
to the batch-averaged gradient. The model is then updated as
$x_{t+1}=x_t-\eta\bar{g}_t$. Formal definitions of differential privacy
and the Gaussian mechanism are provided in Appendix~\ref{app:dp}.

\section{Curvature-Aware Clipping}
\label{sec:theory}

\subsection{Optimal Clipping Geometry}

Vanilla DPSGD utilizes isotropic clipping, projecting gradients onto a
Euclidean ball, which does not account for the anisotropic structure of
the gradient distribution and can lead to unnecessary clipping bias.
To reduce this bias, we generalize the clipping geometry using a bounded
positive definite matrix $G_t \succ 0$. Let
$\|g\|_{G_t}=\sqrt{g^\top G_t g}$. The corresponding clipping operator
with radius $C$ is
$
    \mathrm{clip}_{G_t}(g_{t,\xi},C)
    =
    g_{t,\xi}
    \min\left\{1,\frac{C}{\|g_{t,\xi}\|_{G_t}}\right\}.
$

Let
$b_{G_t}(C)
=
\mathbb{E}_{\xi}
[g_{t,\xi}-\mathrm{clip}_{G_t}(g_{t,\xi},C)]$
denote the clipping bias. We obtain the following bound.

\begin{lemma}[Clipping Bias Bound]
\label{lemma:bias_bound}
For a given clipping radius $C$,
\begin{equation}
    \|b_{G_t}(C)\|_{G_t}
    \le
    \frac{\mathrm{Tr}(G_tF_t)}{C},
\end{equation}
where $F_t$ is the gradient second-moment matrix defined in
Section~\ref{sec:prem}.
\end{lemma}

The proof is provided in Appendix~\ref{app:proof_lemma}.
Directly minimizing $\mathrm{Tr}(G_tF_t)$ yields the trivial solution
$G_t\rightarrow 0$, which reduces the clipping-bias bound at the cost of
increasing the magnitude of the injected noise. In the transformed space,
let $w_t'\sim\mathcal{N}(0,\sigma^2I_d)$. Mapping the noise back to the
original parameter space gives
$\mathrm{Cov}(G_t^{-1/2}w_t')=\sigma^2G_t^{-1}$.
To keep the total noise variance equal to that of isotropic DPSGD, we
impose $\mathrm{Tr}(G_t^{-1})=d$ and formulate
\begin{equation}
\label{eq:optimization_problem}
    \min_{G_t\succ0}
    \quad
    \frac{\mathrm{Tr}(G_tF_t)}{C}
    \qquad
    \mathrm{s.t.}
    \quad
    \mathrm{Tr}(G_t^{-1})=d.
\end{equation}

\begin{theorem}[Optimal Clipping Matrix]
\label{thm:optimal_G}
For $F_t\succ0$, the clipping matrix that minimizes the bias upper bound
in Equation~\eqref{eq:optimization_problem} is
\begin{equation}
    G_t^\star
    =
    \lambda_t F_t^{-1/2},
    \qquad
    \lambda_t
    =
    \frac{\mathrm{Tr}(F_t^{1/2})}{d}.
\end{equation}
\end{theorem}

Theorem~\ref{thm:optimal_G} assigns smaller clipping weights to directions
with larger second moments, thereby adapting the clipping geometry to the
anisotropic gradient distribution. The proof is provided in
Appendix~\ref{app:proof_theorem}.

However, the calculation of $G_t^*$ could become computationally prohibitive in high-dimensional settings and may require additional privacy budget cost, which motivates the scalable diagonal approximation used below.

\subsection{Our Proposed Algorithm: CA-DPSGD}
\label{sec:algorithm}

Algorithm~\ref{alg:cdpsgd} summarizes how CA-DPSGD approximates the
optimal clipping geometry during private training. Instead of directly estimating
the per-sample second-moment matrix $F_t$, CA-DPSGD constructs a diagonal
clipping matrix $G_t$ from previously privatized batch gradients, which
introduces no additional privacy cost by post-processing.

\begin{algorithm}[t]
\small
\caption{Curvature-Aware DPSGD (CA-DPSGD)}
\label{alg:cdpsgd}
\begin{algorithmic}
\REQUIRE $D$, $\eta$, $C$, $\sigma$, $T$, $\beta$.
\FOR{$t=0$ to $T-1$}
    \FOR{each $\xi\in\mathcal{B}_t$}
        \STATE $\tilde g_{t,\xi}
        =\mathrm{clip}_{G_t}(g_{t,\xi},C)$
    \ENDFOR
    \STATE $\tilde g_t
    =\frac{1}{B}\sum_{\xi\in\mathcal{B}_t}\tilde g_{t,\xi}+w_t$,
    $~~w_t\sim\mathcal{N}(0,\sigma^2G_t^{-1})$
    \STATE $x_{t+1}=x_t-\eta\tilde g_t$
    \STATE $\tilde v_{t+1}
    =\beta\tilde v_t+(1-\beta)\tilde g_t^2$
    \STATE $v_{t+1}
    =\left(
    \frac{\sqrt{(\sigma^2/\lambda_t)^2+4\tilde v_{t+1}}
    -\sigma^2/\lambda_t}{2}
    \right)^2$
    \STATE $\lambda_{t+1}
    =\mathrm{Mean}(v_{t+1}^{1/2})$
    \STATE $G_{t+1}
    =\lambda_{t+1}\mathrm{diag}(v_{t+1}^{-1/2})$
\ENDFOR
\RETURN $x_T$
\end{algorithmic}
\end{algorithm}

\noindent\textbf{Curvature-Aware Clipping and Anisotropic Noise.}
CA-DPSGD clips each per-sample gradient according to the current
$G_t$-norm: $ \|g_{t, \xi}\|_{G_t} = \| (G_t)^\frac{1}{2}g_{t, \xi} \|_2$.
This is operationally equivalent to mapping the gradient to the
transformed space, performing Euclidean clipping, and mapping it back.
Accordingly, isotropic Gaussian noise in the transformed space induces
the anisotropic covariance $\sigma^2G_t^{-1}$ in the original parameter
space. The constraint $\mathrm{Tr}(G_t^{-1})=d$ preserves the same total
noise variance as isotropic DPSGD.

\noindent\textbf{Noise-Corrected Second-Moment Estimation.}
Directly using privatized batch gradients to estimate second-order
information introduces the DP noise variance into the estimate.
We therefore maintain an exponential moving average of the squared
privatized gradients, as shown in Algorithm~\ref{alg:cdpsgd}, and
explicitly correct the known contribution of the anisotropic DP noise.
Under the diagonal geometry
$G_t=\lambda_t\mathrm{diag}(v_t^{-1/2})$ and the temporal approximation
$v_t\approx v_{t+1}$, we obtain
\begin{equation}
\label{eq:noise_relation}
    \tilde v_{t+1}
    =
    v_{t+1}
    +
    \frac{\sigma^2}{\lambda_t}\sqrt{v_{t+1}}.
\end{equation}
Solving Equation~\eqref{eq:noise_relation} for its positive root gives
the noise-corrected estimate
\begin{equation}
\label{eq:noise_correction}
    v_{t+1}
    =
    \left(
    \frac{
    \sqrt{(\sigma^2/\lambda_t)^2+4\tilde v_{t+1}}
    -\sigma^2/\lambda_t
    }{2}
    \right)^2.
\end{equation}
The temporal approximation is supported by the stability result in
Appendix~\ref{app:proof_stability}.

\noindent\textbf{Clipping-Geometry Update.}
Finally, enforcing $\mathrm{Tr}(G_{t+1}^{-1})=d$ gives
\begin{equation}
\label{eq:geometry_update}
    \lambda_{t+1}
    =
    \mathrm{Mean}(v_{t+1}^{1/2}),
    \qquad
    G_{t+1}
    =
    \lambda_{t+1}\mathrm{diag}(v_{t+1}^{-1/2}).
\end{equation}

\section{Theoretical Analysis}
\label{sec:proof}

\noindent\textbf{Privacy Guarantee.}
\begin{theorem}
\label{thm:privacy}
Let $q=B/N$ be the sampling probability. There exist constants $c_1$
and $c_2$ such that, for any $\epsilon\le c_1q^2T$ and $\delta>0$,
Algorithm~\ref{alg:cdpsgd} satisfies $(\epsilon,\delta)$-DP if
\begin{equation}
\label{eq:privacy_noise}
    \sigma
    \ge
    c_2\frac{C}{B}
    \frac{q\sqrt{T\log(1/\delta)}}{\epsilon}.
\end{equation}
\end{theorem}
Under $z=G_t^{1/2}g$, the mechanism reduces to Euclidean clipping with isotropic Gaussian noise; adaptation of $G_t$ is post-processing of previously privatized quantities. The proof is provided in Appendix~\ref{app:proof_privacy}.

\noindent\textbf{Convergence Guarantee.}
We assume that $f$ is $L$-smooth, the stochastic-gradient variance is
bounded by $r^2$, and $G_t\succeq\mu I_d$ for some $\mu>0$.
The formal assumptions are provided in Appendix~\ref{app:assumptions}.

\begin{theorem}
\label{thm:convergence}
Under the above assumptions, let
$\eta\le\frac{B}{2(2B+1)L}$. After $T$ iterations,
\begin{align}
\min_{0\le t<T}\mathbb{E}\|\nabla f(x_t)\|_2^2
\le\;&
\frac{4(f(x_0)-f^*)}{\eta T}
+
2L\eta\left(d\sigma^2+\frac{r^2}{B}\right)  +
\frac{4}{\mu C^2T}
\sum_{t=0}^{T-1}\mathrm{Tr}(G_tF_t)^2.
\label{eq:convergence}
\end{align}
\end{theorem}
The bound separates the optimization error, stochastic and DP noise,
and a geometry-dependent clipping-bias term. The last term shows how
the choice of $G_t$ affects clipping bias under anisotropic gradient
second moments. The proof is provided in
Appendix~\ref{app:proof_convergence}.

\begin{table}[t!]
\centering
\small
\renewcommand{\arraystretch}{0.8}
\caption{Test accuracy (\%) comparison on various datasets.}
\label{tab:exp_main}
\resizebox{\textwidth}{!}{
\begin{tabular}{llccccc}
\toprule
Dataset & Model & $(\epsilon, \delta)$ & DPSGD & Auto-S & DP-PSAC & CA-DPSGD \\
\midrule\smallskip
MNIST & 4-layer CNN & $(3, 1\text{e-}5)$ & 98.07 $\pm$ 0.10 & 98.13 $\pm$ 0.04 & 98.17 $\pm$ 0.10 & \textbf{98.32 $\pm$ 0.12} \\\smallskip
Fashion-MNIST & 4-layer CNN & $(3, 1\text{e-}5)$ & 86.45 $\pm$ 0.42 & 86.46 $\pm$ 0.07 & 86.56 $\pm$ 0.13 & \textbf{86.72 $\pm$ 0.18} \\\smallskip
CIFAR-10 & SimCLRv2 & $(2, 1\text{e-}5)$ & 92.62 $\pm$ 0.03 & 92.57 $\pm$ 0.10 & 92.52 $\pm$ 0.12 & \textbf{92.82 $\pm$ 0.09} \\\smallskip
ImageNette & ResNet9 & $(8, 1\text{e-}4)$ & 60.91 $\pm$ 0.71 & 60.97 $\pm$ 0.57 & 61.39 $\pm$ 0.53 & \textbf{62.21 $\pm$ 0.47} \\\smallskip
CelebA & ResNet9 & $(8, 5\text{e-}6)$ & 88.49 $\pm$ 0.13 & 88.61 $\pm$ 0.13 & 88.55 $\pm$ 0.12 & \textbf{88.67 $\pm$ 0.12} \\
\midrule
MNLI & RoBERTa-base & $(8, 1\text{e-}5)$ & 80.43 $\pm$ 0.03 & 80.31 $\pm$ 0.39 & 80.54 $\pm$ 0.16 & \textbf{80.58 $\pm$ 0.27} \\\smallskip
QQP & RoBERTa-base & $(8, 1\text{e-}5)$ & 85.58 $\pm$ 0.27 & 85.40 $\pm$ 0.31 & 85.63 $\pm$ 0.24 & \textbf{85.64 $\pm$ 0.40} \\\smallskip
QNLI & RoBERTa-base & $(8, 1\text{e-}5)$ & 85.79 $\pm$ 0.42 & 86.10 $\pm$ 0.56 & 85.88 $\pm$ 0.51 & \textbf{86.29 $\pm$ 0.65} \\\smallskip
SST-2 & RoBERTa-base & $(8, 1\text{e-}5)$ & 89.38 $\pm$ 0.33 & 89.42 $\pm$ 0.41 & 89.51 $\pm$ 0.36 & \textbf{89.54 $\pm$ 0.29} \\
\bottomrule
\end{tabular}}
\end{table}

\section{Experiments}
\label{sec:exp}

\noindent\textbf{Experimental Setting.}
We evaluate CA-DPSGD on five image classification datasets (MNIST~\citep{deng2012}, Fashion-MNIST~\citep{xiao2017}, CIFAR-10~\citep{krizhevsky2009}, ImageNette~\citep{deng2009imagenet}, and CelebA~\citep{liu2015deep}) and four GLUE tasks~\citep{wang2018} (MNLI, QQP, QNLI, and SST-2). We compare against vanilla DPSGD~\citep{abadi2016}, Auto-S~\citep{bu2024}, and DP-PSAC~\citep{xia2023}. Full experimental settings are provided in Appendix~\ref{app:exp_setting}.

\noindent\textbf{Main Results.}
Table~\ref{tab:exp_main} reports test accuracy across all nine datasets.
CA-DPSGD achieves the highest mean accuracy across the evaluated methods,
with larger improvements on more challenging settings, including gains
of 1.30\% over vanilla DPSGD on ImageNette and 0.50\% on QNLI. These results indicate that adapting the clipping geometry can improve utility across both vision and language tasks.

CA-DPSGD can also be incorporated into existing DPSGD variants,
including IOM~\citep{xiao2023}, KFSGD~\citep{zhang2025}, and
DPAdamBC~\citep{tang2024}. Detailed results are provided in
Appendix~\ref{app:exp_com}. Additional ablations, including the adaptive $\lambda$ mechanism, are provided in Appendix~\ref{app:exp_ab}.

\section{Conclusion}
\label{sec:conclusion}
We proposed CA-DPSGD, a curvature-aware clipping framework that uses
a noise-corrected second-moment estimate to construct a scalable
diagonal clipping geometry with anisotropic Gaussian noise. Theoretical
analysis establishes privacy and convergence guarantees, while
experiments across vision and language tasks show improvements over
the evaluated clipping baselines. A limitation is that the diagonal
geometry does not capture cross-parameter correlations, motivating more expressive clipping geometries in future work.

\bibliography{sample}

@inproceedings{zhang2024,
  author       = {Xinwei Zhang and
                  Zhiqi Bu and
                  Mingyi Hong and
                  Meisam Razaviyayn},
  editor       = {Amir Globersons and
                  Lester Mackey and
                  Danielle Belgrave and
                  Angela Fan and
                  Ulrich Paquet and
                  Jakub M. Tomczak and
                  Cheng Zhang},
  title        = {{DOPPLER:} Differentially Private Optimizers with Low-pass Filter
                  for Privacy Noise Reduction},
  booktitle    = {Annual Conference on Neural Information Processing Systems (NeurIPS)},
  year         = {2024},
  url          = {http://papers.nips.cc/paper\_files/paper/2024/hash/49c466ccc038f39b08b1980a2b06673c-Abstract-Conference.html},
  bibsource    = {dblp computer science bibliography, https://dblp.org},
  address={Vancouver, BC, Canada},
  pages = {1--26},
  publisher={Curran Associates}
}

@inproceedings{xiao2023,
  title={A theory to instruct differentially-private learning via clipping bias reduction},
  author={Xiao, Hanshen and Xiang, Zihang and Wang, Di and Devadas, Srinivas},
  booktitle={IEEE Symposium on Security and Privacy (SP)},
  pages={2170--2189},
  year={2023},
  organization={IEEE},
  publisher={IEEE},
  address={San Francisco, CA, USA} 
}

@article{dwork2014,
  title={The algorithmic foundations of differential privacy},
  author={Dwork, Cynthia and Roth, Aaron and others},
  journal={Foundations and Trends{\textregistered} in Theoretical Computer Science},
  volume={9},
  number={3--4},
  pages={211--407},
  year={2014},
  publisher={Now Publishers, Inc.}
}

@inproceedings{abadi2016,
author = {Abadi, Martin and Chu, Andy and Goodfellow, Ian and McMahan, H. Brendan and Mironov, Ilya and Talwar, Kunal and Zhang, Li},
title = {Deep Learning with Differential Privacy},
year = {2016},
isbn = {9781450341394},
publisher = {Association for Computing Machinery},
address = {New York, NY, USA},
url = {https://doi.org/10.1145/2976749.2978318},
doi = {10.1145/2976749.2978318},
pages = {308–318},
numpages = {11},
location = {Vienna, Austria},
series = {CCS '16},
booktitle = {SIGSAC Conference on Computer and Communications Security}
}

@article{andrew2021,
  title={Differentially private learning with adaptive clipping},
  author={Andrew, Galen and Thakkar, Om and McMahan, Brendan and Ramaswamy, Swaroop},
  journal={Advances in Neural Information Processing Systems},
  volume={34},
  pages={17455--17466},
  year={2021}
}

@inproceedings{yu2021a,
  author       = {Da Yu and
                  Huishuai Zhang and
                  Wei Chen and
                  Tie{-}Yan Liu},
  title        = {Do not Let Privacy Overbill Utility: Gradient Embedding Perturbation
                  for Private Learning},
  booktitle    = {International Conference on Learning Representations (ICLR)},
  publisher    = {ICLR},
  year         = {2021},
  url          = {https://openreview.net/forum?id=7aogOj\_VYO0},
  bibsource    = {dblp computer science bibliography, https://dblp.org},
  address={Virtual},
  pages={}
}

@inproceedings{zhou2020,
  author       = {Yingxue Zhou and
                  Steven Wu and
                  Arindam Banerjee},
  title        = {Bypassing the Ambient Dimension: Private {SGD} with Gradient Subspace
                  Identification},
  booktitle    = {International Conference on Learning Representations (ICLR)},
  publisher    = {ICLR},
  year         = {2021},
  url          = {https://openreview.net/forum?id=7dpmlkBuJFC},
  bibsource    = {dblp computer science bibliography, https://dblp.org},
  pages={1--28},
  address={Virtual}
}

@inproceedings{asi2021,
  title={Private adaptive gradient methods for convex optimization},
  author={Asi, Hilal and Duchi, John and Fallah, Alireza and Javidbakht, Omid and Talwar, Kunal},
  booktitle={International Conference on Machine Learning},
  pages={383--392},
  year={2021},
  organization={PMLR},
  publisher={PMLR},
  address={Virtual}
}

@inproceedings{lee2018,
  title={Concentrated differentially private gradient descent with adaptive per-iteration privacy budget},
  author={Lee, Jaewoo and Kifer, Daniel},
  booktitle={SIGKDD International Conference on Knowledge Discovery \& Data Mining},
  pages={1656--1665},
  year={2018},
  publisher ={ACM},
  address={London, UK}
}

@inproceedings{zhang2023,
  author       = {Xinwei Zhang and
                  Zhiqi Bu and
                  Steven Wu and
                  Mingyi Hong},
  title        = {Differentially Private {SGD} Without Clipping Bias: An Error-Feedback
                  Approach},
  booktitle    = {International Conference on Learning Representations (ICLR)},
  publisher    = {ICLR},
  year         = {2024},
  url          = {https://openreview.net/forum?id=uFbWHyTlPn},
  bibsource    = {dblp computer science bibliography, https://dblp.org},
  address = {Vienna, Austria},
  pages = {1--27}
}

@article{bu2024,
  title={Automatic clipping: Differentially private deep learning made easier and stronger},
  author={Bu, Zhiqi and Wang, Yu-Xiang and Zha, Sheng and Karypis, George},
  journal={Advances in Neural Information Processing Systems},
  volume={36},
  year={2024},
  pages={}
}

@inproceedings{bethune2023,
  author       = {Louis B{\'{e}}thune and
                  Thomas Massena and
                  Thibaut Boissin and
                  Aur{\'{e}}lien Bellet and
                  Franck Mamalet and
                  Yannick Prudent and
                  Corentin Friedrich and
                  Mathieu Serrurier and
                  David Vigouroux},
  title        = {{DP-SGD} Without Clipping: The Lipschitz Neural Network Way},
  booktitle    = {International Conference on Learning Representations (ICLR)},
  publisher    = {ICLR},
  year         = {2024},
  url          = {https://openreview.net/forum?id=BEyEziZ4R6},
  bibsource    = {dblp computer science bibliography, https://dblp.org},
  address={Vienna, Austria},
  pages={}
}

@article{balle2018a,
  title={Privacy amplification by subsampling: Tight analyses via couplings and divergences},
  author={Balle, Borja and Barthe, Gilles and Gaboardi, Marco},
  journal={Advances in neural information processing systems},
  volume={31},
  year={2018},
  pages={}
}

@article{chen2023b,
  title={Differentially private deep learning with dynamic privacy budget allocation and adaptive optimization},
  author={Chen, Lin and Yue, Danyang and Ding, Xiaofeng and Wang, Zuan and Choo, Kim-Kwang Raymond and Jin, Hai},
  journal={IEEE Transactions on Information Forensics and Security},
  year={2023},
  publisher={IEEE},
  pages={4422-4435},
  volume={18},
  number={1}
}

@inproceedings{xia2023,
  title={Differentially private learning with per-sample adaptive clipping},
  author={Xia, Tianyu and Shen, Shuheng and Yao, Su and Fu, Xinyi and Xu, Ke and Xu, Xiaolong and Fu, Xing},
  booktitle={the AAAI Conference on Artificial Intelligence},
  volume={37},
  pages={10444--10452},
  year={2023},
  publisher={AAAI},
  address={Washington, DC, USA}
}

@article{zaheer2018,
  title={Adaptive methods for nonconvex optimization},
  author={Zaheer, Manzil and Reddi, Sashank and Sachan, Devendra and Kale, Satyen and Kumar, Sanjiv},
  journal={Advances in neural information processing systems},
  volume={31},
  year={2018},
  pages={}
}

@article{gorbunov2020,
  title={Stochastic optimization with heavy-tailed noise via accelerated gradient clipping},
  author={Gorbunov, Eduard and Danilova, Marina and Gasnikov, Alexander},
  journal={Advances in Neural Information Processing Systems},
  volume={33},
  pages={15042--15053},
  year={2020}
}

@inproceedings{dwork2006,
  title={Our data, ourselves: Privacy via distributed noise generation},
  author={Dwork, C. and Kenthapadi, K. and McSherry, F. and Mironov, I. and Naor, M.},
  booktitle={International Conference on the Theory and Applications of Cryptographic Techniques},
  pages={486--503},
  year={2006},
  publisher={Springer},
  address={Lyon, France}
}

@article{deng2012,
  title={The mnist database of handwritten digit images for machine learning research [best of the web]},
  author={Deng, Li},
  journal={IEEE signal processing magazine},
  volume={29},
  number={6},
  pages={141--142},
  year={2012},
  publisher={IEEE}
}

@article{xiao2017,
  title={Fashion-mnist: a novel image dataset for benchmarking machine learning algorithms},
  author={Xiao, Han and Rasul, Kashif and Vollgraf, Roland},
  journal={arXiv preprint arXiv:1708.07747},
  year={2017},
  pages={1--6}
}

@techreport{krizhevsky2009,
 author = {Krizhevsky, Alex and Hinton, Geoffrey},
 address = {Toronto, Ontario},
 institution = {University of Toronto},
 number = {0},
 publisher = {Technical report, University of Toronto},
 title = {Learning multiple layers of features from tiny images},
 year = {2009},
 title_with_no_special_chars = {Learning multiple layers of features from tiny images},
 url = {https://www.cs.toronto.edu/~kriz/learning-features-2009-TR.pdf}
}

@inproceedings{zhang2025,
  author       = {Xinwei Zhang and
                  Zhiqi Bu and
                  Borja Balle and
                  Mingyi Hong and
                  Meisam Razaviyayn and
                  Vahab Mirrokni},
  title        = {DiSK: Differentially Private Optimizer with Simplified Kalman Filter
                  for Noise Reduction},
  booktitle    = {The Thirteenth International Conference on Learning Representations,
                  {ICLR} 2025, Singapore, April 24-28, 2025},
  year         = {2025},
  crossref     = {DBLP:conf/iclr/2025},
  url          = {https://openreview.net/forum?id=Lfy9q7Icp9},
  bibsource    = {dblp computer science bibliography, https://dblp.org}
}

@proceedings{DBLP:conf/iclr/2025,
  title        = {The Thirteenth International Conference on Learning Representations,
                  {ICLR} 2025, Singapore, April 24-28, 2025},
  publisher    = {OpenReview.net},
  year         = {2025},
  url          = {https://openreview.net/group?id=ICLR.cc/2025/Conference},
  bibsource    = {dblp computer science bibliography, https://dblp.org}
}

@inproceedings{chen2020a,
  title={A simple framework for contrastive learning of visual representations},
  author={Chen, Ting and Kornblith, Simon and Norouzi, Mohammad and Hinton, Geoffrey},
  booktitle={International conference on machine learning},
  pages={1597--1607},
  year={2020},
  organization={PmLR}
}

@article{wang2018,
  title={GLUE: A multi-task benchmark and analysis platform for natural language understanding},
  author={Wang, Alex and Singh, Amanpreet and Michael, Julian and Hill, Felix and Levy, Omer and Bowman, Samuel R},
  journal={arXiv preprint arXiv:1804.07461},
  year={2018}
}

@article{liu2019,
  author       = {Yinhan Liu and
                  Myle Ott and
                  Naman Goyal and
                  Jingfei Du and
                  Mandar Joshi and
                  Danqi Chen and
                  Omer Levy and
                  Mike Lewis and
                  Luke Zettlemoyer and
                  Veselin Stoyanov},
  title        = {RoBERTa: {A} Robustly Optimized {BERT} Pretraining Approach},
  journal      = {CoRR},
  volume       = {abs/1907.11692},
  year         = {2019},
  url          = {http://arxiv.org/abs/1907.11692},
  eprinttype    = {arXiv},
  eprint       = {1907.11692},
  bibsource    = {dblp computer science bibliography, https://dblp.org}
}

@inproceedings{liu2015deep,
  title={Deep learning face attributes in the wild},
  author={Liu, Ziwei and Luo, Ping and Wang, Xiaogang and Tang, Xiaoou},
  booktitle={Proceedings of the IEEE international conference on computer vision},
  pages={3730--3738},
  year={2015}
}

@inproceedings{deng2009imagenet,
  title={Imagenet: A large-scale hierarchical image database},
  author={Deng, Jia and Dong, Wei and Socher, Richard and Li, Li-Jia and Li, Kai and Fei-Fei, Li},
  booktitle={2009 IEEE conference on computer vision and pattern recognition},
  pages={248--255},
  year={2009},
  organization={Ieee}
}

@inproceedings{sutskever2013,
  title={On the importance of initialization and momentum in deep learning},
  author={Sutskever, Ilya and Martens, James and Dahl, George and Hinton, Geoffrey},
  booktitle={International conference on machine learning},
  pages={1139--1147},
  year={2013},
  organization={pmlr}
}

@inproceedings{kingma2015,
  author    = {Diederik P. Kingma and
               Jimmy Ba},
  title     = {Adam: {A} Method for Stochastic Optimization},
  booktitle = {3rd International Conference on Learning Representations, {ICLR} 2015,
               San Diego, CA, USA, May 7-9, 2015, Conference Track Proceedings},
  editor    = {Yoshua Bengio and
               Yann LeCun},
  year      = {2015},
  url       = {http://arxiv.org/abs/1412.6980},
  bibsource = {dblp computer science bibliography, https://dblp.org}
}

@article{amari1998,
  title={Natural gradient works efficiently in learning},
  author={Amari, Shun-Ichi},
  journal={Neural computation},
  volume={10},
  number={2},
  pages={251--276},
  year={1998},
  publisher={MIT Press}
}

@article{ganesh2023,
  title={Faster differentially private convex optimization via second-order methods},
  author={Ganesh, Arun and Haghifam, Mahdi and Steinke, Thomas and Guha Thakurta, Abhradeep},
  journal={Advances in Neural Information Processing Systems},
  volume={36},
  pages={79426--79438},
  year={2023}
}

@article{krouka2025,
  title={Communication Efficient Differentially Private Federated Learning Using Second Order Information},
  author={Krouka, Mounssif and Koskela, Antti and Kulkarni, Tejas},
  journal={Proceedings on Privacy Enhancing Technologies},
  year={2025}
}

@inproceedings{gupta2018,
  title={Shampoo: Preconditioned stochastic tensor optimization},
  author={Gupta, Vineet and Koren, Tomer and Singer, Yoram},
  booktitle={International Conference on Machine Learning},
  pages={1842--1850},
  year={2018},
  organization={PMLR}
}

@inproceedings{martens2015,
  title={Optimizing neural networks with kronecker-factored approximate curvature},
  author={Martens, James and Grosse, Roger},
  booktitle={International conference on machine learning},
  pages={2408--2417},
  year={2015},
  organization={PMLR}
}

@inproceedings{yao2021,
  title={Adahessian: An adaptive second order optimizer for machine learning},
  author={Yao, Zhewei and Gholami, Amir and Shen, Sheng and Mustafa, Mustafa and Keutzer, Kurt and Mahoney, Michael},
  booktitle={proceedings of the AAAI conference on artificial intelligence},
  volume={35},
  pages={10665--10673},
  year={2021}
}

@inproceedings{tang2024,
  title={Dp-adambc: Your dp-adam is actually dp-sgd (unless you apply bias correction)},
  author={Tang, Qiaoyue and Shpilevskiy, Frederick and L{\'e}cuyer, Mathias},
  booktitle={Proceedings of the AAAI Conference on Artificial Intelligence},
  volume={38},
  pages={15276--15283},
  year={2024}
}

@article{gilani2025,
  title={GeoClip: Geometry-Aware Clipping for Differentially Private SGD},
  author={Gilani, Atefeh and Tasnim, Naima and Sankar, Lalitha and Kosut, Oliver},
  journal={arXiv preprint arXiv:2506.06549},
  year={2025}
}

@book{nocedal2006,
  title={Numerical optimization},
  author={Nocedal, Jorge and Wright, Stephen J},
  year={2006},
  publisher={Springer}
}

@inproceedings{balle2018b,
  title={Improving the gaussian mechanism for differential privacy: Analytical calibration and optimal denoising},
  author={Balle, Borja and Wang, Yu-Xiang},
  booktitle={International conference on machine learning},
  pages={394--403},
  year={2018},
  organization={PMLR}
}
\clearpage
\appendix
\section{Additional Discussion of Related Work}
\label{app:related}

\paragraph{Clipping Bias Mitigation.}
To reduce the bias from gradient clipping, recent works have explored temporal-based methods. Xiao~\textit{et al.}~\citep{xiao2023} proposed to reduce the clipping bias using inner-outer momentum, enhanced network normalization, batch clipping with public data, and data pre-processing. Zhang~\textit{et al.}~\citep{zhang2023} introduced an error-feedback mechanism, DiceSGD, to accumulate the difference between clipped and unclipped gradients. However, maintaining historical buffers significantly increases computational and memory complexity, while error-feedback introduces additional variance as the accumulated error itself must be clipped and perturbed. Bethune~\textit{et al.}~\citep{bethune2023} derive analytical sensitivity bounds using Lipschitz-constrained networks. However, this approach relies on specialized model architectures, limiting its general applicability.

\paragraph{Geometry-Aware and Private Second-Order Methods.}
Geometric information has also been explored in private optimization
~\citep{ganesh2023,tang2024,krouka2025,gilani2025}. However, these methods are limited by restrictive assumptions~\citep{ganesh2023}, hyperparameter instability~\citep{tang2024}, and are infeasible on large datasets~\citep{krouka2025}. GeoClip~\citep{gilani2025}, the closest work to ours, also performs clipping and noise injection in an adaptive geometry estimated from privatized gradients. However, GeoClip minimizes DP noise, whereas we directly minimize an upper bound on clipping bias. We do not include GeoClip in our empirical comparison because its released implementation uses full-rank covariance estimation, which does not scale to the high-dimensional models considered in our experiments.

\section{Differential Privacy Background}
\label{app:dp}
Differential Privacy (DP)~\citep{dwork2014} is a rigorous mathematical technique to protect sensitive information within a dataset $D$, which is formally defined as follows:

\begin{definition}[($\epsilon, \delta$)-Differential Privacy]
A randomized algorithm $\mathcal{M}$ satisfies $(\epsilon, \delta)$-DP if for any two adjacent datasets $D$ and $D'$ that differ by a single element, and for any subset of outputs $S \subseteq \text{Range}(\mathcal{M})$, it holds that:
$$\Pr [\mathcal{M}(D) \in S] \le e^\epsilon \Pr[\mathcal{M}(D') \in S] + \delta, $$
\end{definition}
\noindent where $\epsilon$ and $\delta$ denote the privacy budget and the failure probability, respectively.
To achieve DP, Gaussian noise~\citep{balle2018b} can be added based on the $\ell_2$ sensitivity of the function.

\begin{definition}[Gaussian Mechanism]\label{def.gm}
Suppose an algorithm $\mathcal{A}: \mathcal{D} \to \mathbb{R}^d$ has $\ell_2$ sensitivity $\Delta_\mathcal{A}$. The Gaussian mechanism $\mathcal{M}(D) = \mathcal{A}(D) + w$, with $w \sim \mathcal{N}(0, \sigma^2 I_d)$, is $(\epsilon, \delta)$-DP if the noise $\sigma$ and privacy parameters satisfy:
$$\Phi\left(\frac{\Delta_\mathcal{A}}{2\sigma} - \frac{\epsilon \sigma}{\Delta_\mathcal{A}}\right) - e^\epsilon \Phi\left(-\frac{\Delta_\mathcal{A}}{2\sigma} - \frac{\epsilon \sigma}{\Delta_\mathcal{A}}\right) \le \delta,$$
where $\Phi(t) = \Pr[\mathcal{N}(0, 1) \le t]$ is the cumulative density function of the standard normal distribution and $\Delta_\mathcal{A} = \sup_{D, D'} \|\mathcal{A}(D) - \mathcal{A}(D')\|_2$, for any pair of adjacent datasets $D$ and $D'$.
\end{definition}

\section{Proof of Lemma~\ref{lemma:bias_bound}}
\label{app:proof_lemma}
\begin{proof}
We begin by define $r_{G_t}(g_{t, \xi}) = g_{t, \xi} - \text{clip}_{G_t}(g_{t, \xi}, C)$ and we expand this as: 
\begin{align}
    r_{G_t}(g_{t, \xi}) & = \left( 1 - \min\left(1, \frac{C}{\|g_{t, \xi}\|_{G_t}}\right) \right) g_{t, \xi}  = \left( 1 - \frac{C}{\|g_{t, \xi}\|_{G_t}} \right)_+ g_{t, \xi},
\end{align}
where $(x)_+ = \max(0, x)$.
The magnitude of the bias is $\|b_{G_t}(C)\|_{G_t} = \| \mathbb{E}[r_{G_t}(g_{t, \xi})] \|_{G_t}$. By Jensen's inequality, $\|\mathbb{E}[X]\|_{G_t} \le \mathbb{E}[\|X\|_{G_t}]$, we have:
\begin{align}\label{lemma:bias_bound.temp1}
    \|b_{G_t}(C)\|_{G_t} & \leq \mathbb{E}\left[ \| r_{G_t}(g_{t, \xi}) \|_{G_t} \right] = \mathbb{E}\left[ \left( 1 - \frac{C}{\|g_{t, \xi}\|_{G_t}} \right)_+ \|g_{t, \xi}\|_{G_t} \right].
\end{align}
Let $x = \|g_{t, \xi}\|_{G_t}$. We then prove for any $x \ge 0, C > 0$:
\begin{equation}\label{lemma:bias_bound.temp2}
    \left( 1 - \frac{C}{x} \right)_+ x \leq \frac{x^2}{C}.
\end{equation}
To verify the above inequality, we consider two cases:
\begin{itemize}
    \item Case 1 ($x \le C$): The LHS is 0. Since $x^2/C \ge 0$, the inequality holds.
    \item Case 2 ($x > C$): The LHS is $x - C$. The inequality $x - C \le x^2/C$ can be rearrange to $x^2 - Cx + C^2 \ge 0$. Then, we observe:
    \begin{equation}
        x^2 - Cx + C^2  = (x - \frac{C}{2})^2 + \frac{3}{4}C^2 > 0,
    \end{equation}the inequality holds.
\end{itemize}
Substituting Equation~(\ref{lemma:bias_bound.temp2}) into Equation~(\ref{lemma:bias_bound.temp1}) yields:
\begin{align}
    \|b_{G_t}(C)\|_{G_t} & \leq \mathbb{E}\left[ \frac{\|g_{t, \xi}\|_{G_t}^2}{C} \right] =  \frac{1}{C} \mathbb{E}[g_{t, \xi}^\top {G_t} g_{t, \xi}] =  \frac{\text{Tr}({G_t} \mathbb{E}[g_{t, \xi}g_{t, \xi}^\top])}{C} =  \frac{\text{Tr}(G_t F_t)}{C}.
\end{align}
This completes the proof.
\end{proof}

\section{Proof of Theorem~\ref{thm:optimal_G}}
\label{app:proof_theorem}

\begin{proof}
We aim to minimize the following optimization problem:
\begin{equation}
    \min_{G_t \succ 0} \quad \frac{\text{Tr}(G_t F_t)}{C} \quad \text{s.t.} \quad \text{Tr}(G_t^{-1}) = d.
\end{equation}
We introduce a Lagrange multiplier $\lambda_1 > 0$ for the equality constraint. The Lagrangian function $\mathcal{L}(G_t, \lambda)$ is defined as:
\begin{equation}
    \mathcal{L}(G_t, \lambda_1) = \frac{\text{Tr}(G_t F_t)}{C} + \lambda_1 \left( \text{Tr}(G_t^{-1}) - d \right).
\end{equation}
Then, the gradient is derived as follows:
\begin{align}
    \frac{\partial \mathcal{L}}{\partial G_t} &= \frac{1}{C} \frac{\partial \text{Tr}(G_t F_t)}{\partial G_t} + \lambda_1 \frac{\partial \text{Tr}(G_t^{-1})}{\partial G_t} = \frac{1}{C} F_t - \lambda_1 G_t^{-2}.
\end{align}
Setting the gradient to zero ($\frac{\partial \mathcal{L}}{\partial G_t} = 0$) yields:
\begin{equation}
    \frac{1}{C} F_t = \lambda_1 G_t^{-2}.
\end{equation}
Rearranging the terms to solve for $G_t$, we have:
\begin{equation}
    G_t^2 = \lambda_1 C F_t^{-1}.
\end{equation}
Since $G_t$ and $F_t$ are symmetric positive definite matrices, we obtain:
\begin{equation}
    G_t^* = \sqrt{\lambda_1 C} F_t^{-\frac{1}{2}}.
\end{equation}
Let scalar $\lambda_t = \sqrt{\lambda_1 C}$. The optimal space is given by:
\begin{equation}
    G_t^* = \lambda_t F_t^{-\frac{1}{2}}.
\end{equation}
Substituting $G_t^*=\lambda_tF_t^{-1/2}$ into the constraint
$\mathrm{Tr}((G_t^*)^{-1})=d$ gives
\begin{equation}
\frac{1}{\lambda_t}\mathrm{Tr}(F_t^{1/2})=d,
\end{equation}
and therefore
\begin{equation}
\lambda_t=\frac{\mathrm{Tr}(F_t^{1/2})}{d}.
\end{equation}
which completes the proof.
\end{proof}

\section{Proof of Lemma~\ref{lem:beta_stability}}
\label{app:proof_stability}

To theoretically justify the approximation $\tilde{v}_t \approx \tilde{v}_{t+1}$, we have the following lemma:

\begin{lemma}[Temporal Stability] 
\label{lem:beta_stability}
Under the gradient clipping threshold $C$ and DP noise variance $\sigma^2$, the expected divergence between $\tilde{v}_{t+1}$ and $\tilde{v}_t$ is bounded by:
\begin{equation}
\mathbb{E}[|\tilde{v}_{t+1} - \tilde{v}_t|_1] \le \mathcal{O}\left((1-\beta)(C^2 + d\sigma^2)\right).
\end{equation}
\end{lemma}

\noindent This lemma indicates that as $\beta \to 1$, the difference between $\tilde{v}_{t+1}$ and $\tilde{v}_{t}$ approaches $0$. However, a large $\beta$ could depend too heavily on on historical gradients and potentially slow the convergence. The detailed proof of this lemma is provided in Appendix~\ref{app:proof_stability} and we empirically verified how $\beta$ would affect our performance in Appendix~\ref{app:exp_ps}.
\begin{proof}
Recall the update rule for the privatized second-moment estimator in CA-DPSGD:
\begin{equation}
    \tilde{v}_{t+1} = \beta \tilde{v}_t + (1-\beta) \tilde{g}_{t}^2,
\end{equation}
We aim to bound the expected $\ell_1$-norm of the difference $\Delta_t = \tilde{v}_{t+1} - \tilde{v}_t$.
Substituting the update rule into the difference term, yields:
\begin{align}
    \tilde{v}_{t+1} - \tilde{v}_t 
    &= (\beta - 1)\tilde{v}_t + (1-\beta)\tilde{g}_{t}^2 \notag \\
    &= (1-\beta) (\tilde{g}_{t}^2 - \tilde{v}_t).
\end{align}
Taking the $\ell_1$-norm and the expectation on both sides, and applying the triangle inequality, we have:
\begin{align}
    &\mathbb{E}[\|\tilde{v}_{t+1} - \tilde{v}_t\|_1] \notag \\
    = & (1-\beta) \mathbb{E}[\|\tilde{g}_{t}^2 - \tilde{v}_t\|_1] \notag \\
    \le & (1-\beta) \left( \mathbb{E}[\|\tilde{g}_{t}^2\|_1] + \mathbb{E}[\|\tilde{v}_t\|_1] \right).
\end{align}
Note that for any vector $x$, $\|x^2\|_1 = \|x\|_2^2$. By utilizing the Bounded Clipping Matrix Assumption and the Rayleigh quotient inequality,  we have:
\begin{equation}
    \mathbb{E}[\|\tilde{g}_{t+1}\|_2^2] \le \frac{C^2 + d\sigma^2}{\mu}.
\end{equation}
Since $\tilde{v}_t$ is an exponential moving average of previous noisy gradients, its expected norm is also bounded by:
\begin{equation}
    \mathbb{E}[\|\tilde{v}_t\|_1] \le \max_{\tau \le t} \mathbb{E}[\|\tilde{g}_{\tau}\|_2^2] \le \frac{C^2 + d\sigma^2}{\mu}.
\end{equation}
Substituting these bounds back into the inequality, yields:
\begin{align}
    \mathbb{E}[\|\tilde{v}_{t+1} - \tilde{v}_t\|_1] 
    &\le (1-\beta) \left( (\frac{C^2 + d\sigma^2}{\mu}) + (\frac{C^2 + d\sigma^2}{\mu}) \right) = \frac{2}{\mu}(1-\beta)(C^2 + d\sigma^2).
\end{align}
This completes the proof.
\end{proof}

\section{Privacy Analysis}
\label{app:proof_privacy}
\begin{proof}
    We first define a transformed per-sample gradient as $\tilde{z}_{t, \xi} = (G_t)^\frac{1}{2}\tilde{g}_{t, \xi}$. According to Steps 3 - 4 in Algorithm~\ref{alg:cdpsgd}, we obtain the $l_2$-norm of $\tilde{z}_{t, \xi}$ is bounded by:
    \begin{equation}
        \|\tilde{z}_{t, \xi}\|_2 = \|(G_t)^\frac{1}{2}\tilde{g}_{t, \xi}\|_2 =  \sqrt{\tilde{g}_{t, \xi}^\top {G_t} \tilde{g}_{t, \xi}} = \|\tilde{g}_{t, \xi}\|_{G_t} \leq C.
    \end{equation}
    Then, we obtain that the covariance matrix of the injected DP noise in the transformed space $w_t' = (G_t)^\frac{1}{2}w_t$ is:
    \begin{align}
    \text{Cov}(w_t') =\ & (G_t)^{\frac{1}{2}}(\sigma^2 (G_t)^{-1})((G_t)^{\frac{1}{2}})^T = \sigma^2 I_d.
    \end{align}

\noindent Thus, by utilizing the Gaussian mechanism in Definition~\ref{def.gm} and privacy amplification by subsampling~\citep{balle2018a}, $\tilde{z}_t = \frac{1}{B} \sum_{\xi} \tilde{z}_{t, \xi} + w_t'$ in each training iteration is protected by \((\mathcal{O}(\frac{qC}{B\sigma}, \frac{\delta}{T}))\)-DP. 

Given $G_t$ is derived from the history of previously privatized gradients and $\tilde{g}_t = (G_t)^{-\frac{1}{2}}\tilde{z}_t$, by utilizing post-processing property~\citep{dwork2006}, the noisy batch-averaged gradient $\tilde{g}_t$ is also protected by \((\mathcal{O}(\frac{qC}{B\sigma}, \frac{\delta}{T}))\)-DP. We then apply the moments accountant method~\citep{abadi2016} over \(T\) iterations, and finally obtain that the overall privacy guarantee is \((\epsilon,\delta)\)-DP, which completes the proof.
\end{proof}

\section{Assumptions for Convergence Analysis}
\label{app:assumptions}
\begin{assumption}[$L$-Smoothness]
\label{ass.sm}
The loss function $f(x)$ is a $L$-smooth function, which means that for any $x, y \in \mathbb{R}^d$:
$$\|\nabla f(x) - \nabla f(y)\|_2 \le L \|x - y\|_2.$$
\end{assumption}

\begin{assumption}[Bounded Variance]
\label{ass.var}
The variance of the stochastic gradients is bounded. For any $x \in \mathbb{R}^d$, we have:
$$ \mathbb{E}_{\xi} [\|\nabla f(x; \xi) - \nabla f(x)\|_2^2] \le r^2.$$
\end{assumption}

\begin{assumption}[Bounded Clipping Matrix]
\label{ass.matrix}
$G_t$ is a positive definite matrix with bounded eigenvalues. There exists a constant $\mu > 0$ such that: $$\forall t, \ G_t \succeq \mu I_d.$$
\end{assumption}

The Assumption~\ref{ass.sm} is widely adopted in non-convex optimization~\citep{zaheer2018}. Assumption~\ref{ass.var} is standard in the analysis of gradient clipping~\citep{gorbunov2020}. Assumption~\ref{ass.matrix} is based on our definition of $G_t$, which ensures the matrix is non-singular, and is commonly used in non-private second-order optimizers~\citep{nocedal2006,martens2015}. 

\section{Proof of Theorem~\ref{thm:convergence}}
\label{app:proof_convergence}
We begin with the $L$-smoothness assumption. By substituting $y = x_{t+1}$ and $x = x_t$, we have:
\begin{equation}
    f(x_{t+1}) \le f(x_t) + \langle \nabla f(x_t), x_{t+1} - x_t \rangle + \frac{L}{2} \|x_{t+1} - x_t\|_2^2.
\end{equation}
Using the update rule $x_{t+1} = x_t -\eta \tilde{g}_t$, we get the following.
\begin{equation}
    f(x_{t+1}) \le f(x_t) - \eta \langle \nabla f(x_t), \tilde{g}_t \rangle + \frac{L \eta^2}{2} \|\tilde{g}_t\|_2^2.
\end{equation}
Applying the mathematics expectation and re-arranging terms, we obtain:
\begin{equation}
\label{eq:step1_result}
    \mathbb{E}[f(x_{t+1})] -f(x_t) \le  - \eta \langle \nabla f(x_t), \mathbb{E}[\tilde{g}_t] \rangle + \frac{L \eta^2}{2} \mathbb{E}[\|\tilde{g}_t\|_2^2].
\end{equation}
We then analyze the \textbf{first term} $- \eta \langle \nabla f(x_t), \mathbb{E}[\tilde{g}_t] \rangle$. Using the definition of $\tilde{g}_t$ and the linearity of $\mathbb{E}[\cdot]$, we have:
\begin{align}\label{eq:first_temp}
    \mathbb{E}[\tilde{g}_t] & = \mathbb{E}[\frac{1}{B} \sum_{\xi} \tilde{g}_{t, \xi} + w_t] = \frac{1}{B} \sum_{\xi}\mathbb{E}[\tilde{g}_{t, \xi}] + \mathbb{E}[w_t]  = \frac{1}{B} \sum_{\xi}\mathbb{E}[\tilde{g}_{t, \xi}]  = \nabla f(x_t) + b_t
\end{align}
where the last equality we define $b_t$ as the difference between the expected per-sample clipped gradient and the true gradient:
\begin{equation}\label{eq:b_t}
    b_t = \mathbb{E}[\tilde{g}_{t, \xi}] - \nabla f(x_t).
\end{equation}
Substituting Equation (\ref{eq:first_temp}) into $- \eta \langle \nabla f(x_t), \mathbb{E}[\tilde{g}_t] \rangle$, we have:
\begin{align}\label{eq:first_result}
    - \eta \langle \nabla f(x_t), \mathbb{E}[\tilde{g}_t] \rangle &= - \eta \langle \nabla f(x_t), \nabla f(x_t) + b_t \rangle \notag \\
    &= - \eta \|\nabla f(x_t)\|_2^2 - \eta \langle \nabla f(x_t), b_t \rangle \notag \\
    &\le - \eta \|\nabla f(x_t)\|_2^2 + \frac{\eta}{2} \|\nabla f(x_t)\|_2^2 + \frac{\eta}{2} \|b_t\|_2^2 \notag \\
    &= - \frac{\eta}{2} \|\nabla f(x_t)\|_2^2 + \frac{\eta}{2} \|b_t\|_2^2,
\end{align}
where we apply $-\langle a, b \rangle \le \|a\|\|b\| \le \frac{1}{2}\|a\|^2 + \frac{1}{2}\|b\|^2$ to the inequality.

\noindent For the \textbf{second term} $\frac{L \eta^2}{2} \mathbb{E}[\|\tilde{g}_t\|_2^2]$, given the DP noise is independent, we have:
\begin{equation}\label{eq:second_gt}
    \mathbb{E}[\|\tilde{g}_t\|_2^2] = \mathbb{E}[\|\frac{1}{B} \sum_{\xi}\tilde{g}_{t, \xi}\|^2_2] + \mathbb{E}[\|w_t\|_2^2].
\end{equation}
For term $\mathbb{E}[\|\frac{1}{B} \sum_{\xi}\tilde{g}_{t, \xi}\|^2_2]$, By utilizing $\mathbb{E}[\|X\|^2_2] = \|\mathbb{E}[X]\|^2_2 + \text{Tr}(\text{Var}(X))$, we have:
\begin{align}
    &\mathbb{E}\left[\left\|\frac{1}{B} \sum_{\xi \in \mathcal{B}_t}\tilde{g}_{t, \xi}\right\|^2_2\right]= \left\|\mathbb{E}\left[\frac{1}{B} \sum_{\xi \in \mathcal{B}_t}\tilde{g}_{t, \xi}\right]\right\|^2_2 + \text{Tr}\left(\text{Var}\left(\frac{1}{B} \sum_{\xi \in \mathcal{B}_t}\tilde{g}_{t, \xi}\right)\right).
\end{align}
We analyze the two terms on the RHS separately. For $\mathbb{E}\left[\frac{1}{B} \sum_{\xi}\tilde{g}_{t, \xi}\right]$, since the expectation is a linear operator, we obtain:
\begin{align}
    \mathbb{E}\left[\frac{1}{B} \sum_{\xi}\tilde{g}_{t, \xi}\right] = \frac{1}{B} \sum_{\xi} \mathbb{E}[\tilde{g}_{t, \xi}] = \mathbb{E}[\tilde{g}_{t, \xi}].
\end{align}
Recalling the definition of $b_t$ in Equation (\ref{eq:b_t}), we have:
\begin{align}
    \left\|\mathbb{E}\left[\frac{1}{B} \sum_{\xi}\tilde{g}_{t, \xi}\right]\right\|^2_2 = \|\nabla f(x_t) + b_t\|^2_2.
\end{align}
For the term $\text{Tr}\left(\text{Var}\left(\frac{1}{B} \sum_{\xi \in \mathcal{B}_t}\tilde{g}_{t, \xi}\right)\right)$, since the samples $\xi$ in the mini-batch $\mathcal{B}_t$ are sampled independently, the variance of their sum is the sum of their variances:
\begin{align}
    \text{Var}\left(\frac{1}{B} \sum_{\xi}\tilde{g}_{t, \xi}\right) &= \frac{1}{B^2} \sum_{\xi} \text{Var}(\tilde{g}_{t, \xi}) = \frac{1}{B^2} \cdot B \cdot \text{Var}(\tilde{g}_{t, \xi}) = \frac{1}{B} \text{Var}(\tilde{g}_{t, \xi}).
\end{align}
Taking the trace on both sides:
\begin{align}
    \text{Tr}\left(\text{Var}\left(\frac{1}{B} \sum_{\xi}\tilde{g}_{t, \xi}\right)\right) = \frac{1}{B} \text{Tr}(\text{Var}(\tilde{g}_{t, \xi})).
\end{align}
We then use the property $\text{Tr}(\text{Var}(X)) \le \mathbb{E}[\|X\|^2_2]$, and obtain:
\begin{align}
    \text{Tr}(\text{Var}(\tilde{g}_{t, \xi})) \le \mathbb{E}[\|\tilde{g}_{t, \xi}\|^2_2].
\end{align}
Next, we apply the contraction property of the clipping operation. Since $\|\tilde{g}_{t, \xi}\|_2 = \|\text{clip}(g_{t, \xi})\|_2 \le \|g_{t, \xi}\|_2$, we have:
\begin{align}
    \mathbb{E}[\|\tilde{g}_{t, \xi}\|^2_2] \le \mathbb{E}[\|g_{t, \xi}\|^2_2].
\end{align}
Utilizing the Bounded Variance Assumption, we obtain:
\begin{align}
    \text{Tr}(\text{Var}(\tilde{g}_{t, \xi})) \le \|\nabla f(x_t)\|^2_2 + r^2.
\end{align}
Substituting this back into the variance term equation:
\begin{align}
    \text{Tr}\left(\text{Var}\left(\frac{1}{B} \sum_{\xi}\tilde{g}_{t, \xi}\right)\right) \le \frac{1}{B} (\|\nabla f(x_t)\|^2_2 + r^2).
\end{align}

Combining the Expectation term and the Variance term, we obtain the final upper bound:
\begin{align}\label{eq:second_temp1}
    \mathbb{E}\left[\left\|\frac{1}{B} \sum_{\xi}\tilde{g}_{t, \xi}\right\|^2_2\right] \le \|\nabla f(x_t) + b_t\|^2_2 + \frac{1}{B}(\|\nabla f(x_t)\|^2_2 + r^2).
\end{align}
For term $\mathbb{E}[\|w_t\|_2^2]$, using the constraint $\text{Tr}((G_t)^{-1}) = d$, yields:
\begin{equation}\label{eq:second_temp2}
    \mathbb{E}[\|w_t\|_2^2] = \text{Tr}(\sigma^2 (G_t)^{-1}) = \sigma^2 \text{Tr}((G_t)^{-1}) = d \sigma^2.
\end{equation}
Substituting Equation~(\ref{eq:second_temp1}) and Equation~(\ref{eq:second_temp2}) into Equation~(\ref{eq:second_gt}), we obtain:
\begin{align}\label{eq:second_result}
    &\frac{L \eta^2}{2} \mathbb{E}[\|\tilde{g}_t\|_2^2] \notag \\
    \le & \frac{L \eta^2}{2} \left(\|\nabla f(x_t) + b_t\|^2_2 + \frac{1}{B}(\|\nabla f(x_t)\|^2_2 + r^2)+d \sigma^2\right) \notag \\
    \le & (L \eta^2 + \frac{L \eta^2}{2B}) \|\nabla f(x_t)\|^2_2 + L \eta^2 \|b_t\|^2_2 + \frac{L \eta^2}{2} (d \sigma^2 + \frac{r^2}{B}),
\end{align}
where the last inequality applies $\|a+b\|_2^2 \le 2\|a\|^2_2 + 2\|b\|_2^2.$
Then, substituting Equation~(\ref{eq:first_result}) and Equation~(\ref{eq:second_result}) into Equation~(\ref{eq:step1_result}), we obtain:
\begin{align}
    \mathbb{E}[f(x_{t+1})] - f(x_t) \le & - \frac{\eta}{2} \|\nabla f(x_t)\|_2^2 + \frac{\eta}{2} \|b_t\|_2^2 \notag \\
    & + \left(L \eta^2 + \frac{L \eta^2}{2B}\right) \|\nabla f(x_t)\|^2_2 + L \eta^2 \|b_t\|^2_2 \notag \\
    & + \frac{L \eta^2}{2} \left( d \sigma^2 + \frac{r^2}{B} \right).
\end{align}
Rearranging the terms, we have:
\begin{align}\label{totcon_temp}
    \mathbb{E}[f(x_{t+1})] - f(x_t) \le &  - \left( \frac{\eta}{2} - L \eta^2 \left(1 + \frac{1}{2B}\right) \right) \|\nabla f(x_t)\|^2_2 \notag \\
    & + \frac{L \eta^2}{2} \left( d \sigma^2 + \frac{r^2}{B} \right)+ \left( \frac{\eta}{2} + L \eta^2 \right) \|b_t\|_2^2 .
\end{align}
Given $G_t$ is defined as a bounded and positive definite matrix, there always exists a $\mu > 0$, such that $G_t \succeq \mu I$ holds for all $t$.
Therefore, by using the Rayleigh quotient inequality, we have:
\begin{equation}\label{b_t_bound}
    \|b_t\|_2^2 \le \frac{1}{\mu} \|b_t\|_{G_t}^2 \le \frac{1}{\mu} \left( \frac{\text{Tr}(G_t F_t)}{C} \right)^2.
\end{equation}
Therefore, substituting Equation~(\ref{b_t_bound}) into Equation~(\ref{totcon_temp}), we have:
\begin{align}
    &\mathbb{E}[f(x_{t+1})] - f(x_t) \notag \\
     \le &  - \left( \frac{\eta}{2} - L \eta^2 \left(1 + \frac{1}{2B}\right) \right) \|\nabla f(x_t)\|^2_2 \notag \\
    & + \frac{L \eta^2}{2} \left( d \sigma^2 + \frac{r^2}{B} \right) + \left( \frac{\eta}{2} + L \eta^2 \right) \frac{1}{\mu} \left( \frac{\text{Tr}(G_t F_t)}{C} \right)^2  \notag \\
     \le &  - \frac{\eta}{4}\|\nabla f(x_t)\|^2_2 + \frac{L \eta^2}{2} \left( d \sigma^2 + \frac{r^2}{B} \right)+ \frac{\eta}{\mu} \left( \frac{\text{Tr}(G_t F_t)}{C} \right)^2 ,
\end{align}
where the last inequality applies $\eta \le \min \left(\frac{1}{2L}, \frac{B}{2(2B+1)L}\right) = \frac{B}{2(2B+1)L}.$
Finally, averaging $t$ from 0 to $T-1$, we obtain:
\begin{align}
    &f^* - f(x_0) 
     \le   - \frac{\eta}{4}\sum_{t=0}^{T-1}\|\nabla f(x_t)\|^2_2  + \frac{L \eta^2 T}{2} \left( d \sigma^2 + \frac{r^2}{B} \right)  + \frac{\eta}{\mu C^2} \sum_{t=0}^{T-1}(\text{Tr}(G_t F_t))^2
\end{align}
Rearranging terms and dividing both sides by $\frac{T \eta}{4}$,  we have:
\begin{align}
     \min_t \mathbb{E}[\|\nabla f(x_t)\|^2_2] 
    \le  \frac{4(f(x_0) - f^*)}{\eta T} + 2 L \eta \left( d \sigma^2 + \frac{r^2}{B} \right) 
     + \frac{4\sum_{t=0}^{T-1}(\text{Tr}(G_t F_t))^2}{\mu C^2T},
\end{align}
which completes the proof.

\section{Experimental Setting}
\label{app:exp_setting}
\noindent\textbf{Datasets.} We conduct experiments on five image classification datasets (MNIST~\citep{deng2012}, Fashion-MNIST~\citep{xiao2017}, CIFAR-10~\citep{krizhevsky2009}, ImageNette~\citep{deng2009imagenet}, and CelebA~\citep{liu2015deep}) and four sentence classification datasets from GLUE benchmark~\citep{wang2018} (MNLI, QQP, QNLI, and SST-2). For all datasets, we follow the same data splitting methods as established in prior works~\citep{xia2023,bu2024}.

\noindent\textbf{Baselines.} We compare CA-DPSGD against three baselines: vanilla DPSGD~\citep{abadi2016} as the de facto standard, and two state-of-the-art variants, Auto-S~\citep{bu2024} and DP-PSAC~\citep{xia2023}, which enhance vanilla DPSGD by designing novel clipping functions.

\noindent\textbf{Optimizers.} We utilize SGD with momentum~\citep{sutskever2013} for MNIST, Fashion-MNIST, and CIFAR-10, and Adam~\citep{kingma2015} for all other datasets. 
\section{Additional Experiments: Compatibility with Other DPSGD Variants}
\label{app:exp_com}
\begin{figure}[t!]
    \centering

    \begin{subfigure}[b]{0.4\linewidth}
        \centering
        \includegraphics[width=\linewidth]{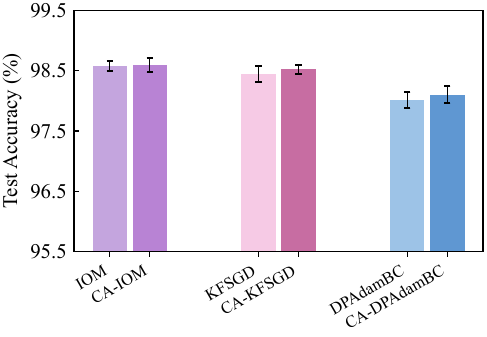}
        \caption{MNIST}
        \label{fig:mnist_comb}
    \end{subfigure}
    \begin{subfigure}[b]{0.4\linewidth}
        \centering
        \includegraphics[width=\linewidth]{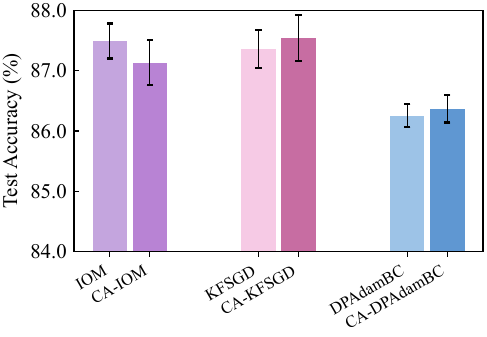}
        \caption{Fashion-MNIST}
        \label{fig:fashionmnist_comb}
    \end{subfigure}
    \\[0.5cm] 
    \begin{subfigure}[b]{0.4\linewidth}
        \centering
        \includegraphics[width=\linewidth]{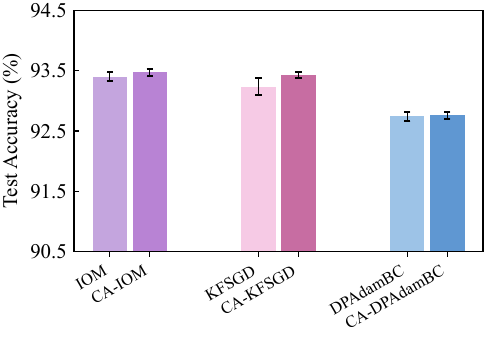}
        \caption{CIFAR-10}
        \label{fig:cifar10_comb}
    \end{subfigure}
    \begin{subfigure}[b]{0.4\linewidth}
        \centering
        \includegraphics[width=\linewidth]{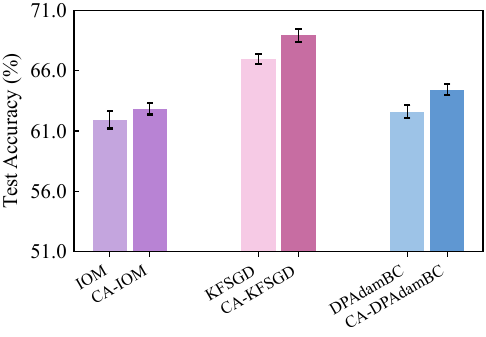}
        \caption{Imagenette}
        \label{fig:imagenette_comb}
    \end{subfigure}
    \\[0.5cm] 
    \begin{subfigure}[b]{0.4\linewidth}
        \centering
        \includegraphics[width=\linewidth]{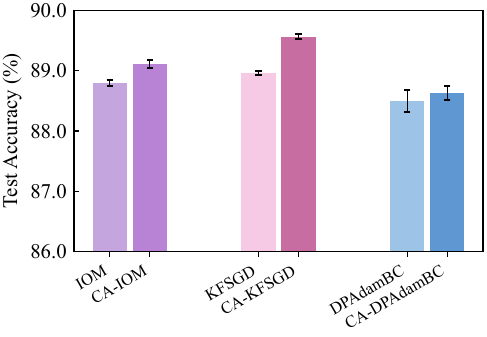}
        \caption{CelebA}
        \label{fig:celeba_comb}
    \end{subfigure}
    
    \caption{Test accuracy (\%) of IOM, KFSGD, and DP-AdamBC with/without CA-DPSGD on five image classification datasets. The prefix "CA-" denotes the method is integrated with CA-DPSGD (e.g., CA-IOM represents IOM combined with CA-DPSGD).}
    \label{fig:all_comb}
\end{figure}

\noindent\textbf{Models and Hyperparameters.} We use the model architectures and hyperparameter setting as in prior works~\citep{xia2023,bu2024} to ensure a fair comparison. Specifically, we use a 4-layer CNN for MNIST and Fashion-MNIST, a pretrained SimCLRv2 model~\citep{chen2020a} for CIFAR-10, and a ResNet9 model for the ImageNette and CelebA dataset. For all datasets in the GLUE benchmark, we fine-tune a pre-trained RoBERTa-base model~\citep{liu2019}. We set the hyperparameter $\beta$ to 0.999.

\noindent\textbf{Hardware and software information.}
All methods are implemented using Python3 and all
experiments are run on an i7-7800X CPU with a NVIDIA RTX A6000
GPU (48 GB). All experiments are repeated five times, with the 
mean and standard deviation reported.

We evaluate the effectiveness of CA-DPSGD as a clipping adapter by integrating it into three state-of-the-art DPSGD variants: (1) IOM~\citep{xiao2023}, which mitigates clipping bias from a temporal perspective via per-sample momentum; (2) KFSGD~\citep{zhang2025}, which reduces DP noise variance using a simplified Kalman filter; and (3) DP-AdamBC~\citep{tang2024}, which improves second-order optimisation using bias correction. 

Figure~\ref{fig:all_comb} illustrates the performance comparison of these methods with and without CA-DPSGD on five image classification datasets, where our method improves the utility of these methods on the image classification datasets in most cases. Specifically, on the Imagenette dataset, the results indicate that CA-DPSGD consistently enhances the utility of all three methods, improving the test accuracy of KFSGD and DP-AdamBC by approximately 2.0\% and 1.8\%, respectively. This implies that our method further mitigates the clipping bias, which is orthogonal to the improvements made by these variants.

However, CA-IOM exhibits a marginal performance drop compared to vanilla IOM on Fashion-MNIST. This is because momentum smoothing in the IOM amplifies the contribution of low-curvature dimensions in such a simple landscape. In contrast, CA-DPSGD prioritises high-curvature dimensions by relaxing their thresholds while enforcing stricter clipping on low-curvature dimensions. Consequently, this results in information loss in the low-curvature subspace, introducing an additional error that outweighs the benefits of geometric rectification.

\section{Additional Experiments: Ablation Study}
\label{app:exp_ab}

\begin{figure*}[h!] 
\centering
    \begin{subfigure}[b]{0.8\linewidth}
        \centering
        \includegraphics[width=\linewidth]{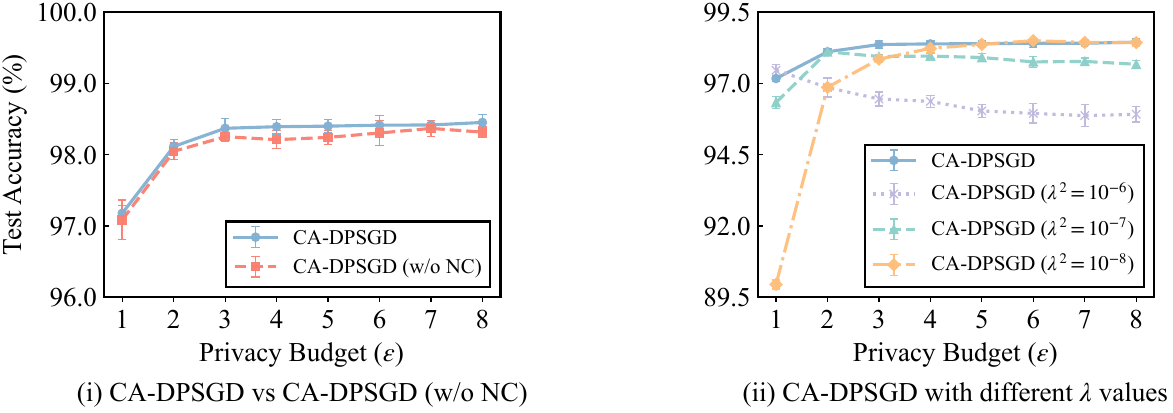}
        \caption{MNIST}
        \label{fig:appendix_ablation_mnist}
    \end{subfigure}
    \\[0.5cm]  
    \begin{subfigure}[b]{0.8\linewidth}
        \centering
        \includegraphics[width=\linewidth]{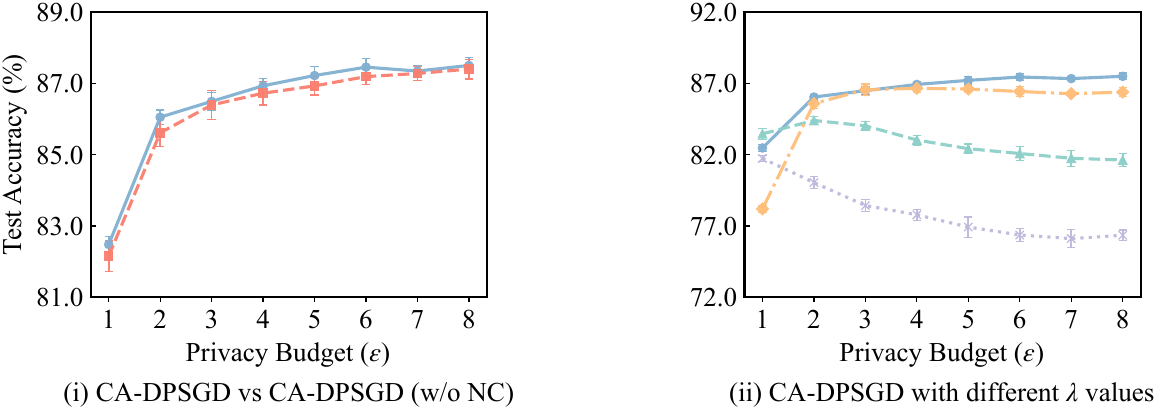}
        \caption{Fashion-MNIST}
        \label{fig:appendix_ablation_fmnist}
    \end{subfigure}
    \\[0.5cm] 
    \begin{subfigure}[b]{0.8\linewidth}
        \centering
        \includegraphics[width=\linewidth]{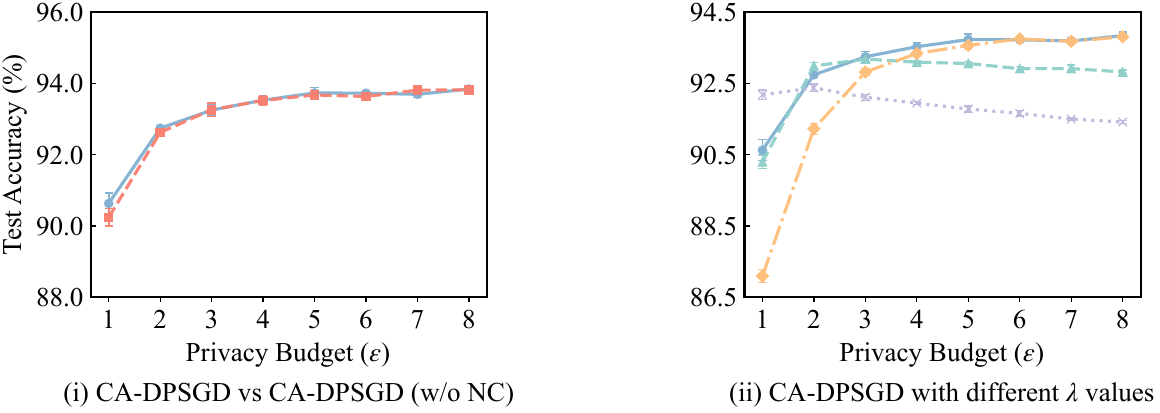}
        \caption{CIFAR-10}
        \label{fig:ppendix_ablation_cifar10}
    \end{subfigure}
    
    \caption{Ablation study on test accuracy (\%) of CA-DPSGD on three image classification datasets: MNIST, Fashion-MNIST, CIFAR-10. (a) Comparison of CA-DPSGD and CA-DPSGD without Noise Correction (w/o NC). (b) Comparison between the proposed adaptive $\lambda$ mechanism and fixed $\lambda$ values ($\lambda^2 \in \{10^{-6}, 10^{-7}, 10^{-8}\}$).}
    \label{fig:all_ablation}
\end{figure*}

As shown in the left column of Figure~\ref{fig:all_ablation}, the Noise Correction (NC) module improves the utility on MNIST and Fashion-MNIST in most cases. However, the performance gap between CA-DPSGD and its variant w/o NC is negligible on CIFAR-10. This is because we fine-tune a pre-trained model on CIFAR-10, a setting where the injected noise variance is small. Consequently, the error introduced by the DP noise is minimal, making the NC module less critical in this specific setting.

In contrast, the right column of Figure~\ref{fig:all_ablation} demonstrates the superiority of our adaptive $\lambda$ mechanism over fixed values ($\lambda^2 \in \{10^{-6}, 10^{-7}, 10^{-8}\}$). This performance difference stems from the fundamental bias-variance trade-off in DPSGD. Specifically, a large $\lambda$ (e.g., $10^{-6}$) suppresses the variance of the injected DP noise, yielding relative good performance at low $\epsilon$ regimes. However, it introduces an excessive clipping bias that hinders convergence as $\epsilon$ increases. In contrast, a small $\lambda$ reduces clipping bias but suffers from a high DP variance. Consequently, our adaptive approach achieves the best performance in most scenarios, reducing the need for manual hyperparameter tuning across diverse datasets and varying privacy budgets.

\section{Additional Experiments: Parameter Study of \texorpdfstring{$\beta$}{beta}}
\label{app:exp_ps}
\begin{figure*}[t]
    \centering
    \begin{subfigure}[b]{0.7\linewidth}
        \centering
        \includegraphics[width=\linewidth]{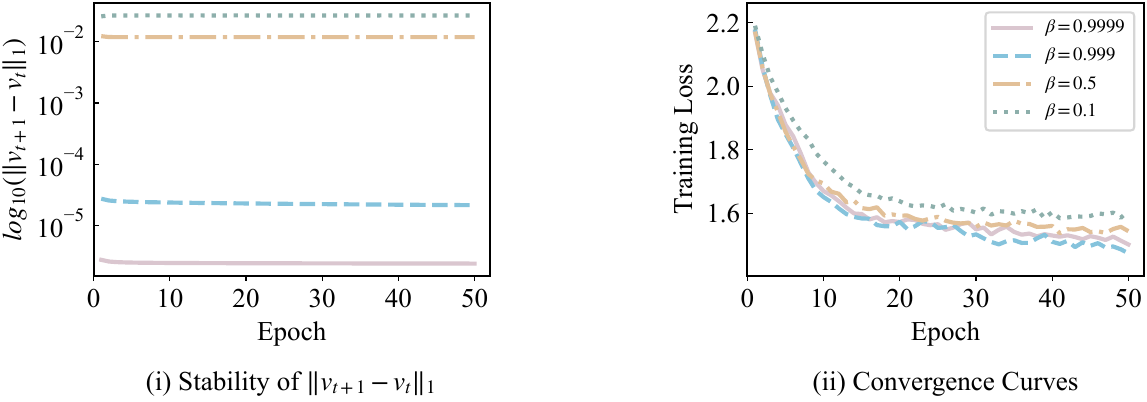}
        \caption{Imagenette}
        \label{fig:beta_imagenette}
    \end{subfigure}
    \vspace{1em}
    \begin{subfigure}[b]{0.7\linewidth}
        \centering
        \includegraphics[width=\linewidth]{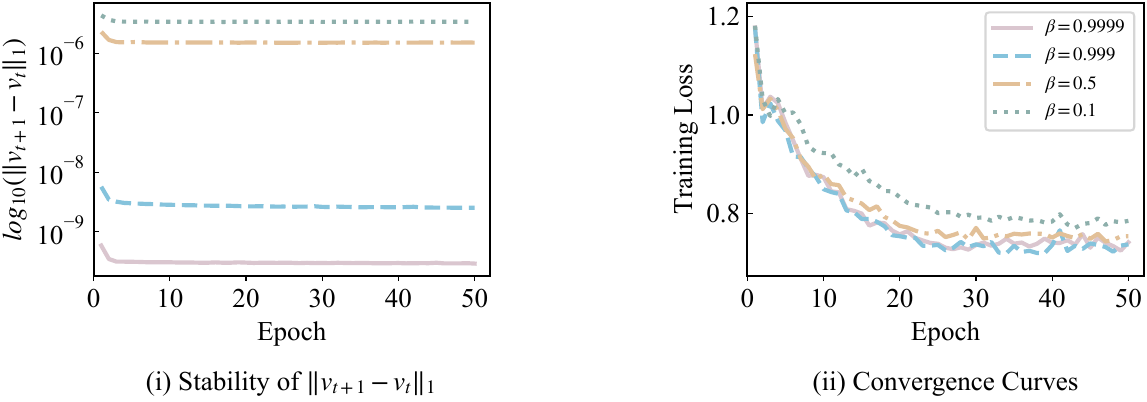}
        \caption{Fashion-MNIST}
        \label{fig:beta_fashionmnist}
    \end{subfigure}

    \caption{Hyperparameter analysis of $\beta$ on different datasets. (a) Results on Imagenette and (b) Results on Fashion-MNIST. In both cases, the left panels show the stability of $v_t$ via $\log_{10}(\|v_{t+1} - v_t\|_1)$, while the right panels show training loss curves. Larger $\beta$ leads to more stable momentum updates but slower convergence.}
    \label{fig:combined_beta_ablation}
\end{figure*}

The divergence of $v_t$ is theoretically bounded by $\mathcal{O}(1-\beta)$, implying that a larger $\beta$ enforces stricter stationarity. This theoretical insight is empirically illustrated in Figure~\ref{fig:beta_imagenette}(i), where increasing $\beta$ significantly reduces the difference $\|v_{t+1} - v_t\|_1$. Smaller $\beta$ values (e.g. $0.1$ and $0.5$) induce excessive difference between $v_{t+1}$ and $v_t$), thus degrading the convergence performance as shown in Figure~\ref{fig:beta_imagenette}(ii).

However, excessive stability could also hinder optimization. Figure~\ref{fig:beta_imagenette}(ii) reveals that $\beta=0.9999$ leads to slower convergence compared to $\beta=0.999$. Therefore, we utilize $\beta=0.999$ as an effective trade-off in our experiments, ensuring sufficient stability for the correction mechanism while maintaining efficient convergence. 

Figure~\ref{fig:beta_fashionmnist} shows trends consistent with our previous analysis: larger $\beta$ enhances stability, while smaller $\beta$ accelerates convergence. For Fashion-Mnist, $\beta=0.9999$ eventually achieves a performance comparable to $\beta=0.999$, despite a slower convergence speed. This is because the simplicity of the Fashion-MNIST landscape allows effective convergence.

\section{Notations}
\label{app:notation}
We summarize the notation in Table~\ref{tab.notations}.

\begin{table}[t!]
\centering
\caption{Table of Notations}
\label{tab.notations}
\begin{tabular}{|c|c|}
\hline
\multicolumn{1}{|c|}{\textbf{Notation}} & \multicolumn{1}{c|}{\textbf{Description}} \\
\hline
$\epsilon$ & Privacy budget\\
\hline
$\delta$ & Failure factor \\
\hline
$x_t$ & Model parameters at iteration $t$ \\
\hline
$d$ & Dimensionality of model parameters \\
\hline
$N$ & Total number of samples in dataset \\
\hline
$B$ & Batch size \\
\hline
$\eta$ & Learning rate \\
\hline
$C$ & Clipping threshold \\
\hline
$\sigma^2$ & DP noise variance \\
\hline
$f(x)$ & Loss function \\
\hline
$g_{t, \xi}$ & Per-sample gradient for sample $\xi$ \\
\hline
$\tilde{g}_{t, \xi}$ & Clipped per-sample gradient \\
\hline
$\tilde{g}_t$ & Noisy batch-averaged gradient \\
\hline
$F_t$ & Fisher Information Matrix \\
\hline
$G_t^*$ & Optimal curvature-aware clipping matrix \\
\hline
$\lambda_t$ & Adaptive scaling factor \\
\hline
$v_t$ & Rectified second-order moment estimate \\
\hline
$\tilde{v}_t$ & Noisy second-order moment (EMA) \\
\hline
$\beta$ & EMA factor \\
\hline
$L$ & Lipschitz smoothness constant \\
\hline
$r^2$ & Upper bound on gradient variance \\
\hline
\end{tabular}
\end{table}

\end{document}